\documentclass{article}[11pt]
\usepackage{graphicx} % Required for inserting images

\usepackage{amsmath}
\usepackage{amssymb} 
\usepackage{authblk}

\usepackage[commands,environments,theorems]{AVT}

\usepackage[round]{natbib}

\usepackage{amsmath}
\usepackage{subfigure}
\RequirePackage{amssymb}
\RequirePackage{amsthm}

\DeclareMathOperator{\KL}{\mathsf{KL}}
\DeclareMathOperator{\TV}{\mathsf{TV}}

\DeclareMathOperator{\Id}{\mathrm{Id}}

\DeclareMathOperator{\Vol}{\mathrm{Vol}}

\DeclareMathOperator{\dist}{\operatorname{dist}}
\DeclareMathOperator{\diam}{\operatorname{diam}}

\def\abs#1{\left|#1\right|}
\def\norm#1{\|#1\|}
\def\curly#1{\{#1\}}
\def\square#1{\left[#1\right]}

\def\v#1{\boldsymbol{#1}}

\hypersetup{colorlinks,citecolor=blue}

\newcommand{\expectation}{\mathbb{E}}

\newcommand{\cN}{\mathcal{N}}

\newcommand*\samethanks[1][\value{footnote}]{\footnotemark[#1]}

\usepackage{graphicx}
\usepackage{enumitem}
\usepackage{natbib}
\usepackage{amsmath,amsfonts,amssymb,amsthm}
\usepackage{mathtools}
\usepackage[margin=3cm]{geometry}
\usepackage{hyperref}
\usepackage{cleveref}
\usepackage{xcolor}
\usepackage{upgreek}

\usepackage{hyperref}
\usepackage{xcolor}

\hypersetup{
     colorlinks=true,
     citecolor=blue,  % Color for citation numbers
    linkbordercolor=black, % Color for reference boxes  
}

\title{Linear Convergence of Diffusion Models Under the Manifold
Hypothesis}
\author{Peter Potaptchik\thanks{Equal contribution}\,}
\author{Iskander Azangulov\samethanks \,}
\author{George Deligiannidis}
\affil{University of Oxford\\  \{surname\}@stats.ox.ac.uk}

\date{}

\begin{document}

\maketitle

\begin{abstract}
Score-matching generative models have proven successful at sampling from complex high-dimensional data distributions. In many applications, this distribution is believed to concentrate on a much lower $d$-dimensional manifold embedded into $D$-dimensional space; this is known as the manifold hypothesis. The current best-known convergence guarantees are either linear in $D$ or polynomial (superlinear) in $d$. The latter exploits a novel integration scheme for the backward SDE. We take the best of both worlds and show that the number of steps diffusion models require in order to converge in Kullback-Leibler~(KL) divergence is linear (up to logarithmic terms) in the intrinsic dimension $d$. Moreover, we show that this linear dependency is sharp. 
\end{abstract}

\section{Introduction}
Score-matching generative models~\cite{ho2020denoising,song2021scorebased} such as diffusion models have become a leading paradigm for generative modeling. They achieve state-of-the-art results in many domains including audio/image/video synthesis~\cite{evans2024fasttimingconditionedlatentaudio,dhariwal2021diffusionmodelsbeatgans, ho2022videodiffusionmodels}, molecular modeling~\cite{watson2023molecule}, and recently text generation~\cite{pmlr-v235-lou24a}. Informally, diffusion models take samples from a distribution in $\R^D$, gradually corrupt them with Gaussian noise, and then learn to reverse this process. Once trained, a diffusion model can turn noise into new samples from the data distribution by iteratively applying the denoising procedure.

Due to the empirical success of diffusion models, there has been a push~\cite{oko2023,wibisono2024optimalscoreestimationempirical, wu2024theoreticalinsightsdiffusionguidance} to better understand their theoretical properties, in particular, their convergence guarantees. 

An important question is to determine the \textit{iteration complexity} of diffusion models. In this paper, this refers to the number of steps diffusion models require in order to converge in Kullback-Leibler divergence to the original distribution regularized by a small amount of Gaussian noise. 
Assuming only the existence of a second moment, \cite{benton2024nearly} prove that the iteration complexity is at most linear (up to logarithmic factors) in $D$. 

While this result is tight in the general case, many real-world distributions actually have a low-dimensional structure. The assumption that a distribution lives on a $d$-dimensional manifold is called the \textit{manifold hypothesis}, see e.g. \cite{bengio2013representation}. This hypothesis has been supported by empirical evidence in many settings, e.g. image data, in which diffusion models are particularly successful. Therefore, the study of diffusion models under this assumption has garnered increased interest~\cite{kadkhodaie2024generalizationdiffusionmodelsarises, pmlr-v238-tang24a,pidstrigach2022}. 

Recently~\cite{li2024} have shown that there is a special discretization design guaranteeing an iteration complexity of $d^4$ in the intrinsic dimension~$d$. The current~\cite{azangulov2024convergencediffusionmodelsmanifold} best-known bound scales as $d^3$.

\subsection*{Our Contribution} In this work, we improve upon these results and show that the number of steps diffusion models require to converge in
KL divergence is \textit{linear} (up to logarithmic terms) in the \textit{intrinsic} dimension $d$. This is formalized in \ref{thm:main}. Additionally, we prove that the linear dependency is sharp. 

The proof follows the structure of \cite{chen2023samplingeasylearningscore} and \cite{benton2024nearly} combined with a result from~\cite{azangulov2024convergencediffusionmodelsmanifold} providing bounds on the score function depending only on the intrinsic dimension $d$. A key insight of our proof exploits the inherent martingale structure in diffusion processes. As we show, with the right SDE discretization, the corresponding error can be represented as a sum of easy-to-control martingale increments leading to a very concise argument. 

We posit that this scaling is one of the major reasons why diffusion models are able to perform so well on tasks such as synthetic image generation. While the extrinsic dimension of image datasets is very large, e.g. $\approx 1.5\cdot 10^5$ for ImageNet, ~\cite{pope2021intrinsicdimensionimagesimpact} estimate that the true intrinsic dimension is much lower, e.g. around $50$ for ImageNet. Our result implies that the number of steps diffusion models need to sample scales as the latter rather than the former. This helps explain why diffusion models are able to generate crisp image samples with less than $1000$ iterations~\cite{ho2020denoising}. We note that this is not the discretization generally used in practice; however, a single-line modification to existing implementations would result in this discretization scheme.

In Section~\ref{section:preliminaries}, we give an overview of diffusion models, our assumptions and the discretization scheme that we use. We introduce our main result in Section~\ref{section:main_results} and give its proof in Section~\ref{section:proof}. We conclude in Section~\ref{section:conclusion}. 

\paragraph{Concurrent work:} Two weeks after the first version of our work was made publicly available, an independent work appeared by ~\cite{huang2024denoisingdiffusionprobabilisticmodels} who derive similar bounds. In addition, they relax some of our assumption on the support and boundedness of the measure.

\section{Preliminaries} \label{section:preliminaries}
\subsection{Diffusion Models}
Suppose we want to generate samples from a distribution $\mu$ on $\R^D$. Diffusion models solve this by first specifying a forward noising process $\curly{X_t}_{t \in [0,T]}$ up to some time $T$. This process is defined as the evolution of data $X_0 \sim \mu$ according to an Ornstein-Uhlenbeck~(OU)~SDE
\begin{equation*}
    \begin{cases}    
    dX_t = -X_t dt + \sqrt{2}dB_t,& t\in (0,T]
    \\
    X_0 \sim \mu,
    \end{cases}
\end{equation*}
where $\curly{B_t}_{t \in [0,T]}$ is a Brownian motion on $\R^D$. Letting $c_t := \exp(- t)$ and $\sigma^2_t := 1-\exp(-2t)$, we note that $X_t \stackrel{dist.}{=} c_t X_0 + \sigma_tZ_D$ where $Z_D \sim \mathcal{N}(0, \Id_D)$. We use $p_t$ to denote the marginal density of $X_t$.

    The reverse process $\curly{Y_t}_{t \in [0,T]} := \curly{X_{T-t}}_{t \in [0,T]}$, under mild assumptions~\cite{ANDERSON1982313}, satisfies
\begin{equation}
\label{eqn:reverse_sde}
\begin{cases}
    dY_t = (Y_t + 2 \nabla \log p_{T-t}(Y_t))dt + \sqrt{2}dB_t', & t\in (0,T]
    \\
    Y_0 \sim p_T,
\end{cases}
\end{equation}
where $\curly{B_t'}_{t \in [0,T]}$ is another Brownian motion on $\R^D$. By generating samples $Y_0 \sim p_T$ and then simulating~\eqref{eqn:reverse_sde} up to time $T$, we can obtain samples $Y_T \sim \mu$ from the data distribution. The main idea behind diffusion models is to simulate these dynamics approximately since neither $p_T$ nor the score function $s_t := \nabla \log p_t$ are known. 

In practice, we solve these problems by building an approximate process $\hat{Y}_t$. Due to the exponential convergence of the OU process to a standard normal distribution, for a sufficiently large $T$, we have that $p_T \approx \mathcal{N}(0, \Id_D)$. So, we initialize as $\hat{Y}_0 \sim \mathcal{N}(0, \Id_D)$. Second, we learn a score approximation $\hat{s}_t \approx s_t$ which is used instead of the true score function. The approximation $\hat{s}_t$ is usually parameterized by neural networks and trained via a denoising score-matching objective. 

Finally, we note that the SDE~\eqref{eqn:reverse_sde} cannot be simulated exactly and instead a discretization scheme must be introduced. More precisely, first, in order to ensure numerical stability, an early stopping time $\delta>0$ is chosen. Next, the interval $[0,T-\delta]$ is divided into $K$ time steps $0=t_0 < t_1 < \ldots < t_K = T-\delta$. The final discretization is given by
\begin{align} \label{eqn:discretized_sde}
    \begin{cases}    
    \hat{Y}_{t_{k+1}} = \alpha_{k}\hat Y_{t_{k}} + \beta_{k}\hat{s}_{T-t_k}(\hat Y_{t_{k}}) + \eta_{k}Z_k, & 0\le k < K
    \\
    \hat{Y}_{t_0} \sim \cN\left(0,\Id_D\right),
    \end{cases}
\end{align} 
where $Z_k \stackrel{i.i.d.}{\sim} \mathcal{N}(0, \Id_D)$ and $\alpha_k, \beta_k, \eta_k$ are real numbers. The discretization schedule $t_0 < t_1 < \ldots < t_K$ and coefficients $\alpha_k,\beta_k,\eta_k$ are hyper-parameters specifying the design of the sampling procedure. We discuss the choice of the discretization coefficients and their importance in Section~\ref{sub:discretization_scheme}, and we specify the discretization schedule $\{t_k\}$ that gives a linear in $d$ convergence rate in Section~\ref{section:main_results}.

\begin{figure}[ht]
    \centering
    \subfigure[Hilbert curves of orders $n= 2,4,6$]{
        \includegraphics[width=0.45\textwidth]{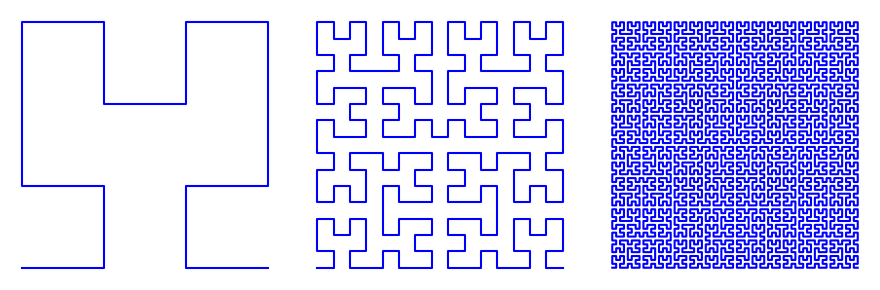}
        \label{fig:first_fig}
    }
    \subfigure[Premature early stopping leads to over-smoothing of the manifold structure]{
        \includegraphics[width=0.34\textwidth, height=0.15\textwidth]{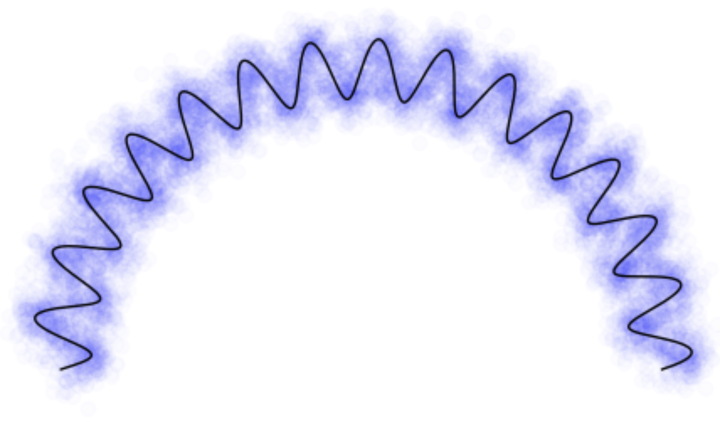}
        \label{fig:second_fig}
    }
\end{figure}

\subsection{Assumptions and Notation}
Throughout the paper, we assume that the distribution $\mu$ satisfies the manifold hypothesis which we state more formally in the following.
\begin{assumption} \label{asmp:1}
    $\mu$ is supported on a smooth, compact, $d$-dimensional, $\beta\ge 2$-smooth manifold $M$ embedded into $\R^D$.
\end{assumption}
\noindent We also make the following assumption for ease of presentation. The general case is handled by rescaling and shifting.
\begin{assumption}
\label{asmp:2}
    We assume $\diam M \le 1$ and $0\in  M$. 
\end{assumption}
Without any restrictions on the manifold $M$, we cannot hope to obtain bounds independent of the ambient dimension $D$. As an intuitive counter-example, consider $D$-dimensional Hilbert curves $M_n$, see Figure~\ref{fig:first_fig}. In the limit, these curves cover the entire $D$ dimensional cube $M_\infty := [0,1]^D$. Moreover, any measure on $M_{\infty}$ can be seen as a weak limit of measures on $M_n$. This makes sampling from $M_n$ (for large enough $n$) as hard as sampling from $M_{\infty}$, which according to~\cite[Appendix~H]{benton2024nearly} scales as at least $D$.

Therefore, we should place additional assumptions on the complexity of $\mu$ and $M$ in order to avoid such pathological cases. 
Informally (see details in Appendix~\ref{apdx:0}), the complexity of $M$ depends both on its global~(volume) and local~(smoothness) properties. We control its smoothness by introducing a scale $r > 0$ at which $M$ is locally flat.
To control the measure $\mu$, we assume that it has a density $p_0(dy)$ (w.r.t. the standard volume form $dy$) bounded from above and below. We assume logarithmic control over the discussed quantities.
\begin{assumption} \label{asmp:3}
    There is a constant $C > 0$ such that $\Vol M \le e^{dC}$, $r > e^{-C}$, and $e^{-dC} \le p_0 \le e^{dC}$.
\end{assumption}
Note that we impose different assumptions on $r$ and $\Vol M$, this follows from a relation between the radius and volume of the $d$-dimensional sphere, i.e. ${\Vol B_d(0,r) \propto r^d}$.
\begin{remark}
    Informally, on scales larger than $r$, the manifold $M$ is not flat anymore. So, if $M$ is corrupted by noise of magnitude greater than $r$, over-smoothing may destroy the geometric structure, see Figure~\ref{fig:second_fig}. 
    Diffusion models stopped at time $\delta$ add Gaussian noise proportional to $\sqrt{\delta}$. 
    To capture the shape of $M$ %shouldn't be corrupted by noise of magnitude greater than $r$.  So, 
    the stopping time should be chosen to satisfy $\sqrt{\delta} \lesssim r$. Thus we should expect $\log \delta^{-1} \ge C$.     
\end{remark}

Recall that $\delta$ denotes the early stopping time, and $0 = t_0 < t_1 < \ldots < t_K = T-\delta$ are the $K$ discretization time steps. Let $\gamma_k := t_{k+1} - t_{k}$ be the $k$-th step size. We control the score estimation error of $\hat{s}_t$ as follows.
\begin{assumption} \label{asmp:score}
    The score network $\hat{s}_t(x)$ satisfies
\begin{align*}    
    \sum_{k = 0}^{K-1}\gamma_k \E\|s_{T-t_k}(X_{T-t_k}) - \hat{s}_{T-t_k}(X_{T-t_k})\|^2 
    \leq \varepsilon_{score}^2.
\end{align*}
\end{assumption}

\section{Discretization Scheme} \label{sub:discretization_scheme}
    In this section, we describe a discretization scheme that allows for polynomial convergence in the intrinsic dimension $d$. 
    The most popular discretization scheme~\eqref{eqn:discretized_sde} is the so called \textit{exponential integrator}. This involves a naive approximation of~\eqref{eqn:reverse_sde} by 
    \begin{equation}
    \label{eqn:reverse_sde_approx}
        \begin{cases}
            d\hat{Y}_t = (\hat{Y}_t + 2 \nabla \hat{s}_{T-{t_k}}(\hat{Y}_{t_k}))dt + \sqrt{2}dB_t', & t\in [t_k,t_{k+1})
            \\
            \hat{Y}_0 \sim p_T,
        \end{cases}
    \end{equation}
    where the score part of the drift is made piecewise-constant. The discretization coefficients $\alpha_k, \beta_k, \eta_k$ are then chosen so that~\eqref{eqn:discretized_sde} corresponds to solving~\eqref{eqn:reverse_sde_approx} exactly. 
    
    However, as pointed out in \cite{li2024}, the iteration complexity of the exponential integrator grows with $D$, and, actually, there is a unique setting of the discretization coefficients in~\eqref{eqn:discretized_sde} that yields an iteration complexity independent of $D$. It is given by
    \begin{equation} 
    \label{eq:discrete_version_of_final_SDE}
        \begin{cases}    
            \hat{Y}_{t_{k+1}} = c^{-1}_{\gamma_k}\hat{Y}_{t_k} 
            + \frac{\sigma^2_{\gamma_k}}{c_{\gamma_k}} \hat{s}_{T-t_k}(\hat{Y}_{t_k})
            + \sigma_{\gamma_k}\frac{\sigma_{T-t_{k+1}}}{\sigma_{T-t_k}}Z_k, 
            & Z_k \stackrel{i.i.d.}{\sim} \mathcal{N}(0, \Id_D) 
            \\
            \hat{Y}_{t_0} \sim \cN\left(0,\Id_D\right).
        \end{cases}
    \end{equation} 
    \begin{remark}
        This discretization scheme corresponds to the DDPM update rule~\cite{ho2020denoising}.
    \end{remark}
    In Appendix~\ref{apdx:discretization}, we provide a simple argument that derives~\eqref{eq:discrete_version_of_final_SDE} as a correction of~\eqref{eqn:discretized_sde}. We summarize it below. We recall Tweedie's formula~\cite{robbins1956} stating that
    \begin{equation}    
    \label{eq:tweedie_formula}
    s_t(x) = \frac{c_t\expectation[X_0|X_t=x] - x}{\sigma^2_t} = -\frac{\expectation[Z_D|X_t=x]}{\sigma_t},
    \end{equation}  
    for $X_t = c_tX_0 +\sigma_t Z_D$, where $X_0\sim \mu$ and $Z\sim \cN\left(0, \Id_D\right)$ are independent. We recall \cite[Theorem 15]{azangulov2024convergencediffusionmodelsmanifold}, which gives the following.
    \begin{theorem}
    \label{thm:main_bound_on_score}
    Let $\mu$ be a measure satisfying Assumptions~\ref{asmp:1}--\ref{asmp:3}. Let $t > 0$ and let $X_t = c_tX_0 +\sigma_t Z_D$, where $X_0\sim \mu$ and $Z_D\sim \cN\left(0, \Id_D\right)$ are independent. Then with probability at least $1-\delta$
    \[
    \norm{\sigma_t s(t,X_t) + Z_D} \lesssim \sqrt{d\left(C+\max\{\log(c_t/\sigma_t), 0\}\right) + \log \delta^{-1}}.
    \] 
    \end{theorem}
    In particular, substituting~\eqref{eq:tweedie_formula} into  Theorem~\ref{thm:main_bound_on_score} and using that $\diam M \le 1$ 
    \[
    \expectation[\norm{\expectation[X_0|X_t] -X_0}^2] \lesssim d\left(C+\max\{\log(c_t/\sigma_t), 0\}\right).
    \]
    Thus, we propose to approximate the score $s_{t}(x_{t})$ for $t < t'$ by the following  
    \begin{equation}    
    \label{eqn:first_order}
    s_{t}(x_{t}|t',x_{t'}) := \frac{c_{t}\expectation[X_0|X_{t'}=x_{t'}] - x_{t}}{\sigma^2_{t}} = c^{-1}_{t'-t}\frac{\sigma^2_{t'}}{\sigma^2_{t}}s_{t'}(x_{t'}) - \frac{x_{t}-c^{-1}_{t'-t} x_{t'}}{\sigma^2_{t}}.
    \end{equation}
    By Theorem~\ref{thm:main_bound_on_score}, this approximation induces an error that only scales with the intrinsic dimension $d$.  We analogously define $\hat{s}_t(x_t|t',x_{t'})$ to be a correction of the score estimate $\hat{s}_t$ obtained by substituting $\hat{s}_{t'}(x_{t'})$ instead of $s_{t'}(x_{t'})$ into~\eqref{eqn:first_order}. 
    \begin{restatable}{proposition}{EQUIVALENCETOLIE}    
    \label{prop:equiv_to_lie}
    The scheme~\eqref{eq:discrete_version_of_final_SDE} can be obtained by solving the following continuous-time dynamics
    \begin{equation} \label{eq:approx_backward_process_discrete_modified}
        \begin{cases}
            d\hat{Y}_t = \square{\hat{Y}_t + 2\hat{s}_{T-t}(\hat{Y}_t|T-t_k,\hat{Y}_{t_k})}dt + \sqrt{2}dB_t', 
            & t\in [t_k,t_{k+1})
            \\
            \hat{Y}_0 \sim \cN\left(0, \Id_D\right).
        \end{cases}
    \end{equation}
    \end{restatable}  
    This can be proven by integrating the linear SDE \eqref{eq:approx_backward_process_discrete_modified}. See~\eqref{eqn:reason} and then a discussion in Remark~\ref{remark:benton_proof} for additional properties of this discretization scheme. 

\section{Main Results} \label{section:main_results}
    Before stating our main result let us introduce the discretization schedule, first introduced in~\cite{benton2024nearly} which we use in what follows. %Fix a positive $\kappa < 1/4$ and choose a partition $0=t_0 < t_1 <\ldots < t_K = T-\delta$ such that $\gamma_k = t_{k+1}-t_k \le \kappa\min\left(1, T-t_k\right)$. 
    \begin{definition} \label{definition:discretization}
        Fixing integers $L < K$, and choose $T := \kappa L +1$ and $\delta := (1+\kappa)^{L-K}$, one can take the uniform partition $t_{k} = \kappa k, k < L$ of $[0, T-1]$ and the exponential partition 
        $t_{L+m} 
        = T - (1+\kappa)^{-m}$ of $[T-1, T-\delta]$ for $0 \leq m \leq K-L$. 
    \end{definition}
    The main result of our paper is the following.
    \begin{theorem} \label{thm:main}
        Let $\mu$ be a measure satisfying assumptions~\ref{asmp:1}--\ref{asmp:3} and let $\hat{s}_t$ be a score approximation satisfying Assumption~\ref{asmp:score} with $\{t_k\}_{k\le K}$ satisfying definition~\ref{definition:discretization}. Then the process $\hat{Y}_t$ following~\eqref{eq:discrete_version_of_final_SDE} satisfies
        \begin{align} \label{eq:main_result}
            \KL(\hat{Y}_{T-\delta}\| X_{\delta}) \lesssim \varepsilon^2_{score} + De^{-2T} 
            + \kappa + d\kappa^2(K-L)\left)(\log \delta^{-1} + C\right),
        \end{align}
        where we use $\KL(X\|Y)$ to denote the $\KL$-divergence between the laws of $X$ and $Y$. 
    \end{theorem}
    
    This bound consists of four terms: (i) $\varepsilon^2_{score}$ corresponding to the error in score approximation; (ii) $De^{-2T}$ corresponding to the initialization error; (iii) $\kappa$ corresponding to the discretization error for $t \in[0, T-1]$; (iv) $d\kappa^2(K-L)(\log \delta^{-1}+C)$ corresponding to the discretization error for $t \in[T-1, T-\delta]$. By choosing $K$ to be sufficiently large, we obtain a linear (up to logarithmic factors) in $d$ bound on the iteration complexity in the following.
    \begin{corollary}
        Under the same assumptions as in Theorem~\ref{thm:main}, for a given $\kappa < 1/4$ and tolerance $\varepsilon > 0$, choosing $L \simeq \kappa^{-1}\left(\log D + \log \varepsilon^{-1}\right)$ and $K-L \simeq \kappa^{-1}\log \delta^{-1}$, diffusion models with $K$ denoising steps achieve an error bounded as
        \[
        \KL(\hat{Y}_{T-\delta}\| X_{\delta}) \lesssim  \varepsilon^2 + \varepsilon^2_{score} + \frac{d(\log \delta^{-1}+\log \varepsilon^{-1} + \log D)(\log \delta^{-1}+C)\log\delta^{-1}}{K}.
        \]
    \end{corollary}
    \paragraph{Tightness of linear bound.} In the best case, a~diffusion model learns the score function exactly and is initialized at the true initial distribution.
    Let $\pi$ be a compactly supported distribution on $\R$ such that a diffusion model with a given discretization scheme and perfect score/initialization achieves an $\bar \varepsilon$ error in $\KL$. 
    
    Consider the product measure $\pi^{\otimes d}\otimes \left(\delta_{0}\right)^{\otimes (D-d)}$ on $\R^d\times \R^{D-d}$. Note that its score function is
    $(s_t(x_1),\ldots, s_t(x_d), -\sigma^{-2}_t x_{d+1}, \ldots, -\sigma^{-2}_t x_{D})$. So, by the tensorization of $\KL$, assuming perfect score estimation and initialization for $\pi^{\otimes d}\otimes \left(\delta_{0}\right)^{\otimes (D-d)}$, the same discretization scheme as above will result in a $d \bar \varepsilon$ error. Therefore a linear dependence in $d$ is optimal. 

    \section{Proof of Theorem~\ref{thm:main}} \label{section:proof}
    We defer the proofs of lemmas to Appendix~\ref{apdx:proofs}. The main idea of the proof of Theorem~\ref{thm:main} is to exploit the martingale structure of $\E[X_0|X_t]$, which we now formalize. We introduce the functions $m_t(x) := \E[X_0 | X_t=x]$ and define a stochastic process $\curly{m_{t}(X_{t})}_{t\in [0,T]}$. Note that $m_t(X_t) = \E[X_0|X_t]$. The next result states that this process is a martingale. 
    % Note the choice of filtration.
     \begin{restatable}{lemma}{ReverseMartingale}
    Define the filtration $\mathcal{F}_t := \sigma(X_{s} : s \ge t)$. The process $\curly{m_t(X_t)}_{t \in [0,T]}$ is a martingale w.r.t. $\{\mathcal{F}_t\}_{t \in [0,T]}$. In particular, for $t < t'$
    \begin{align}
        \E[m_t(X_t)|X_{t'}] = m_{t'}(X_{t'}).
    \end{align}
    \end{restatable}

    \noindent By Assumption~\ref{asmp:1}, $X_0$ is a.s.\ bounded. Therefore the process $\curly{m_t(X_t)}_{t \in [0,T]}$ is square integrable. So, the following lemma on the orthogonality of martingale increments can be applied.
    \begin{restatable}{lemma}{MartingaleIncrements}
    \label{lemma:orth_increments}
        Let $\{M_t\}_{t \geq 0}$ be a square integrable martingale in $\R^D$ w.r.t.\ a filtration $\{\mathcal{F}_t\}_{t \geq 0}$. Then for any $t_1 < t_2 < t_3$ we have 
        \begin{align*}
            \E \|M_{t_3}-M_{t_{1}}\|^2 = \E \|M_{t_3}-M_{t_{2}}\|^2 + \E \|M_{t_2}-M_{t_{1}}\|^2. 
        \end{align*}
    \end{restatable}

    \noindent We are now ready to present the proof of Theorem~\ref{thm:main}. Our first steps coincide with \cite{chen2023improvedanalysis, benton2024nearly}. We begin by decoupling the errors coming from score estimation, approximate initialization and SDE discretization. This is formalized by the following lemma.
    \begin{restatable}{lemma}{Decoupling}
    Under the same assumptions as in Theorem~\ref{thm:main}
    \begin{align} \label{eq:chen_bound_of_KL}
        \KL(\hat{Y}_{T-\delta}\| X_\delta) \lesssim \varepsilon^2_{score} +  D e^{-2T} + \sum_{k = 0}^{K-1} \int^{t_{k+1}}_{t_{k}}\expectation \norm{s_{T-t}(X_{T-t}\mid T-t_k, X_{T-t_k}) - s_{T-t}(X_{T-t})}^2dt.
    \end{align}
    \end{restatable}

    \noindent Note that we get we get the first and second terms of the desired bound in~\eqref{eq:main_result} plus the discretization error of the ideal score approximation. So, it is sufficient to bound this latter sum by $\kappa + d\kappa^2 (K-L)(\log \delta^{-1}+C)$.

    The next observation is that \eqref{eq:tweedie_formula} combined with \eqref{eqn:first_order} gives
    \begin{align} \label{eqn:reason}
        s_{t}(x|t',x') - s(t,x) = \frac{c_t}{\sigma^2_t}\left(\expectation[X_0|X_t=x] - \expectation[X_0|X_{t'}=x']\right) = \frac{c_t}{\sigma^2_t} \left(m_t(x) - m_{t'}(x')\right).
    \end{align}
    This implies that
    \begin{align}
    \label{eq:tweedi_corollary}
    \expectation \norm{s_{t}(X_t|t',X_{t'}) - s_t(X_t)}^2 = \frac{c^2_t}{\sigma^4_t}\expectation \norm{m_t(X_t) - m_{t'}(X_{t'})}^2.
    \end{align}
    \noindent Applying Lemma~\ref{lemma:orth_increments} with $t_1 =0, t_2 = t, t_3 = t'$ we get
    \begin{align}\label{eqn:helpful_martingale_increment}
        \expectation{\|m_0(X_0) - m_{t'}(X_{t'})\|}^2 = \expectation{\|m_0(X_0) - m_{t}(X_{t})\|}^2 + \expectation{\|m_t(X_t) - m_{t'}(X_{t'})\|}^2.
    \end{align}
    Since $m_0(X_0) = \expectation[X_0|X_0] = X_0$, substituting~\eqref{eqn:helpful_martingale_increment} into~\eqref{eq:tweedi_corollary} gives    \begin{align}\label{eqn:discretization_control}
        \E\norm{s_{t}(X_t|t',X_{t'}) - s_t(X_t)}^2 = \frac{c^2_t}{\sigma^4_t}\left(\expectation{\|X_0 - m_{t'}(X_{t'})\|}^2 
         - \expectation{\|X_0 - m_{t}(X_{t})\|}^2\right).
    \end{align}
    \noindent We next formalize the intuitive statement that the discretization error increases with the time gap.
    \begin{restatable}{lemma}{Monotonicity}\label{lemma:monotonicity} 
        For $t_1 < t_2 < t'$
            \[
         \expectation \norm{s_{t_1}(X_{t_1}|t',X_{t'}) - s_{t_1}(X_{t_1})}^2 \ge  \expectation \norm{s_{t_2}(X_{t_2}|t',X_{t'}) - s_{t_2}(X_{t_2})}^2.
        \]
    \end{restatable}

    \noindent We can now use Lemma~\ref{lemma:monotonicity} and then~\eqref{eqn:discretization_control} to bound the sum in~\eqref{eq:chen_bound_of_KL} as
     \begin{align}       
         &\sum_{k = 0}^{K-1} \int^{t_{k+1}}_{t_{k}}\expectation \norm{s_{T-t}(X_{T-t}\mid T-t_k, X_{T-t_k}) - s_{T-t}(X_{T-t})}^2dt \notag
         \\
         &\qquad \le \sum_{k = 0}^{K-1} (t_{k+1}-t_k)\expectation \norm{s_{T-t_{k+1}}(X_{T-{t_{k+1}}}\mid T-t_k, X_{T-t_k}) - s_{T-t_{k+1}}(X_{T-t_{k+1}})}^2 \notag
            \\
          &\qquad = \sum_{k = 0}^{K-1} (t_{k+1}-t_k)\frac{c^2_{T-t_{k+1}}}{\sigma^4_{T-t_{k+1}}}\left(\expectation{\|X_0 - m_{T-t_k}(X_{T-t_k})\|}^2 
         - \expectation{\|X_0 - m_{T-t_{k+1}}(X_{T-t_{k+1}})\|}^2\right). \label{eq:main_bound_part_1}
    \end{align}
    Next, we will split the sum in~\eqref{eq:main_bound_part_1} into two terms: (i) the sum for $t_k \in [0, T-1]$ and (ii) the sum for $t_k \in [T-1, T-\delta]$. In particular, the first term will sum over indices $k = 0$ to $k = L-1$, and the second term will sum from $k = L$ to $k = K-1$. The first term (i) can be bounded by a telecoping argument as follows. We recall that $L$ was chosen in Definition~\ref{definition:discretization} so that $t_L = T - 1$ and $t_{k+1}-t_k = \kappa$ for $k < L$. Therefore for $k \leq L$, we have that $({c^2_{T-t_{k}}}/{\sigma^4_{T-t_{k}}}) \leq 1 / \sigma^4_1 \leq 4$. So, by telescoping we obtain
    \begin{align} 
        & \sum_{k = 0}^{L-1} (t_{k+1}-t_k)\frac{c^2_{T-t_{k+1}}}{\sigma^4_{T-t_{k+1}}}\left(\expectation{\|X_0 - m_{T-t_k}(X_{T-t_k})\|}^2 - \expectation{\|X_0 - m_{T-t_{k+1}}(X_{T-t_{k+1}})\|}^2\right) \notag \\
        & \qquad \leq 4 \kappa\expectation{\|X_0 - m_{T}(X_{T})\|}^2 \le 4\kappa. \label{eqn:t_small}
    \end{align}
    The last inequality follows from Assumption~\ref{asmp:1} combined with
    \[
    \norm{X_0 - m_t(X_t)} = \norm{\int_M (X_0-y)\mu_{0|t}(dy|X_t)} \le  \sup_{y\in M}{\norm{X_0-y}} \le \diam M \le 1,
    \]
    where $\mu_{0|t}(dy|x)$ is the law of $\left(X_0|X_t=x\right)$. Next, we deal with the the second term (ii) which corresponds to the terms $k = L$ to $k = K-1$. This term is bounded using the exponential partitioning of time gaps $\gamma_k$ in $[T-1, T-\delta]$. By the choices in Definition~\ref{definition:discretization}, we have $T-t_k \le 1$, and so $(t_{k+1}-t_k)/\sigma^4_{T-t_{k+1}} \leq  4(t_{k+1}-t_k)/(T-t_{k+1})^2 \le 8\kappa/(T-t_{k+1})$. Also recalling that $c_t \leq 1$ for all $t$, we obtain
    \begin{align} 
     \lefteqn{\sum_{k = L}^{K-1} (t_{k+1}-t_k)\frac{c^2_{T-t_{k+1}}}{\sigma^4_{T-t_{k+1}}}\left(\expectation{\|X_0 - m_{T-t_k}(X_{T-t_k})\|}^2 
     - \expectation{\|X_0 - m_{T-t_{k+1}}(X_{T-t_{k+1}})\|}^2\right) } \notag
     \\
     &\leq 8\kappa\sum_{k = L}^{K-1} \frac{1}{T-t_{k+1}}\left(\expectation{\|X_0 - m_{T-t_k}(X_{T-t_k})\|}^2 
     - \expectation{\|X_0 - m_{T-t_{k+1}}(X_{T-t_{k+1}})\|}^2\right) \notag
     \\
     &\le
     8\kappa \expectation{\|X_0 - m_{1}(X_{1})\|}^2 + 8\kappa\sum_{k = L+1}^{K-1} \frac{t_{k+1} - t_{k}}{\left(T-t_{k+1}\right)\left(T-t_{k}\right)}\expectation{\|X_0 - m_{T-t_k}(X_{T-t_k})\|}^2 \notag
     \\
     &\le
     8\kappa + 16\kappa^2\sum_{k = L+1}^{K-1} \frac{1}{\left(T-t_{k}\right)}\expectation{\|X_0 - m_{T-t_k}(X_{T-t_k})\|}^2. \label{eqn:t_large}
    \end{align}
    Finally, we bound $\E{\|X_0 - m_{T-t_k}(X_{T-t_k})\|}^2$ via the following lemma.
    \begin{restatable}{lemma}{ManifoldConcentrationBound} \label{lemma:manifold_conc_bound}
    Let $\mu$ satisfy Assumptions~\ref{asmp:1}--\ref{asmp:3}. Fix positive $\delta < 1/4$. Then for any $t > \delta$
    \[
    % \label{eq:manifold_eq}
    \expectation{\|X_0 - m_{t}(X_{t})\|}^2 \lesssim \min\square{1, dt\left(\log \delta^{-1} + C\right)}.
    \]
    \end{restatable}

    \noindent Combining~\eqref{eqn:t_small}~and~\eqref{eqn:t_large} with Lemma~\ref{lemma:manifold_conc_bound}, we obtain the following desired inequality
    \[
    \sum_{k = 0}^{K-1} \int^{t_{k+1}}_{t_{k}}\expectation \norm{s_{T-t}(X_{T-t}\mid T-t_k, X_{T-t_k}) - s_{T-t}(X_{T-t})}^2dt \lesssim \kappa + d\kappa^2(K-L)\left(\log \delta^{-1} + C\right).
    \]

    \begin{remark} \label{remark:benton_proof}
    We return to explaining the nature of first-order correction in \eqref{eqn:first_order}. In the proof, we control the discretization error of the backwards SDE by controlling the difference in drifts $s_t(x|t',x') - s_t(x)$, see \eqref{eq:chen_bound_of_KL}. By construction, the differences are proportional to martingale increments $m_t(x)-m_{t'}(x)$ as in \eqref{eqn:reason}. This allows us to use the orthogonality of martingale increments after that. \\

\cite{benton2024nearly} also leverage martingale properties of the score function. They note that $M_{t} := e^{-(T-t)} s_{T-t}(X_{T-t})$ is a martingale. However, since they consider a standard exponential integrator scheme, the discretization error which is given by the difference in drifts $s_{T-t}(X_{T-t})-s_{T-t_k}(X_{T-t_k})$ is a linear combination of a martingale increment $\left(M_t-M_{t_k}\right)$ and the score term $s_{T-t}(X_{T-t})$. The last term scales as the norm of a $D$-dimensional Gaussian noise vector. We avoid this problem by adjusting the discretization coefficients to kill the second term. So, the difference is only a scaled martingale increment. This enables bounds that are independent of $D$.   
    \end{remark}
\section{Conclusion \& Future Work} \label{section:conclusion}
    In this work, we studied the iteration complexity of diffusion models under the manifold hypothesis. Assuming that the data is supported on a $d$-dimensional manifold, we proved the first linear in $d$ iteration complexity bound w.r.t.\ $\KL$ divergence. Furthermore, we showed that this dependence is optimal. 

    This result is equivalent to an $O(d / K)$ convergence rate in $\KL$ where $K$ is the number of discretization steps. A simple application of Pinkser's inequality gives an $O(\sqrt{d / K})$ bound for the total variation ($\TV$) distance as well. 
% Acknowledgments---Will not appear in anonymized version
\section*{Acknowledgements}
    GD was supported by the Engineering and Physical Sciences Research
    Council [grant number EP/Y018273/1].
    IA was supported by the Engineering and Physical Sciences Research Council [grant number
    EP/T517811/1].  
    PP is supported by the EPSRC CDT in Modern Statistics and Statistical Machine Learning (EP/S023151/1)
    
\bibliographystyle{plainnat}
\bibliography{references}

\appendix

% \Crefalias{section}{appendix} % uncomment if you are using cleveref

\section{Elements of Manifold Learning}
\label{apdx:0}
    We follow \cite[Section 2.2]{azangulov2024convergencediffusionmodelsmanifold} in defining the class of regular smooth manifolds of interest. We give only the bare minimum details required to describe it. For a more comprehensive discussion, see \cite{divol2022measure}. 
    
    We recall that a $d$-dimensional manifold~\cite{lee2013introduction} is a topological space $M$ that is locally isomorphic to an open subset of $\R^d$. In other words, for each $y\in M$, there is an open set $y\in U_y\subseteq M$ and a continuous function $\Phi_y:U_y\rightarrow \R^d$ such that $\Phi_y$ is a homeomorphism onto its image. The smoothness of the manifold is defined as the smoothness of the functions $\Phi_y$. 

    When a manifold $M$ is embedded into $\R^D$, a key quantity~\cite{federer} used to control the regularity of the embedding is called the \emph{reach} $\tau = \tau(M)$ and is defined as 
    \begin{align*}
            \tau &:= \sup\curly{\varepsilon :  \forall x\in M^\varepsilon,\,\,\,\, \exists!\, y\in M \, \text{ s.t. } \dist(x,M) =\|x-y\|},
    \end{align*}
    where $M^\varepsilon = \curly{x\in \R^D: \dist(x,M) < \varepsilon}$ is the $\varepsilon$-neighborhood of $M$.
    Equivalently, $\tau$ is the supremum over the radii of neighborhoods of $M$ for which the projection is unique. The reach controls~\cite{divol2022measure} both the global and local properties of the manifold.

    In particular, the reach controls the scale at which $M$ admits a natural smooth parameterization $\Phi_y$ which will be used to control the smoothness of the manifold. More precisely, for a point $y\in M$, let $\pi_y:= \pi_{T_yM}$ be the orthogonal projection onto the tangent space $T_yM\simeq \R^d$ at $y$. Then~\cite{aamari2018stability} the restriction of $\pi_y$ to $M\cap B_{\R^D}(y,\tau/4)$ is one-to-one and $B_{T_yM}(0,\tau/8) \subseteq \pi_y (M\cap B_D(y,\tau/4))$. Defining $\Phi_y$ as the inverse of $\pi_y\big|_{M\cap B_{\R^D}(y,\tau/4)}$, we have constructed a local parameterization $\Phi_y: B_{T_yM}(0,\tau/8) \rightarrow M$ of the manifold $M$ at the point $y$.   

    We assume that the $\Phi_y$ are in $C^2(B_{T_yM}(0,\tau/8))$, and define $L := L(M) := \sup_y \norm{\Phi_y}_{C^2(B_{T_yM}(0,\tau/8))}$. From a geometric perspective, this allows us to compare tangent vectors at different points by applying parallel transport which is defined in terms of a second-order differential operator. Finally, we define $r$ in Assumption~\ref{asmp:3} as $r:= \min\left(\tau,L\right)/8$.
\section{Discretization Coefficients}  
\label{apdx:discretization}
    \begin{figure}
        \centering
        \includegraphics[width=0.9\textwidth]{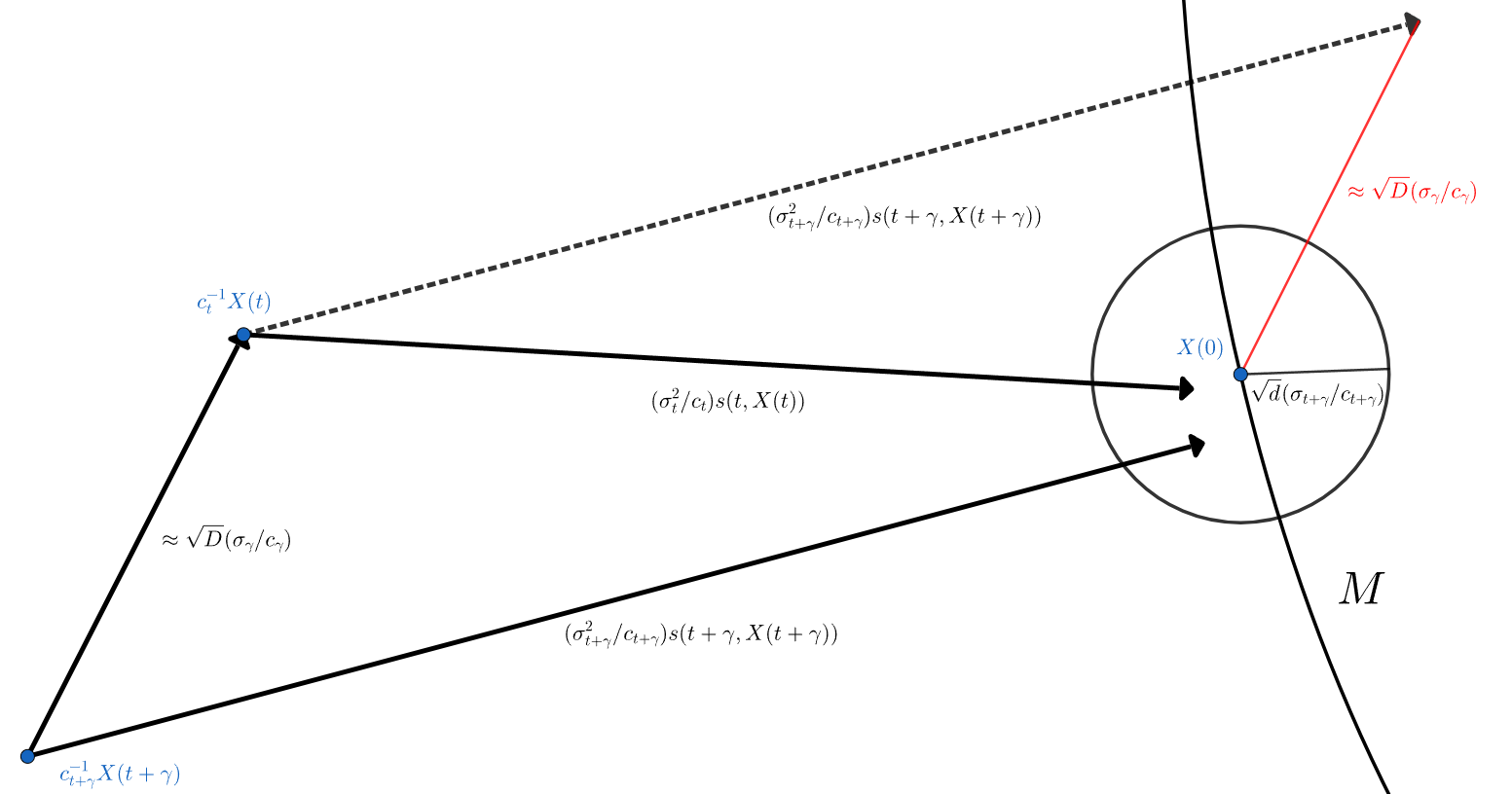}
        \caption{Illustration of a typical error of classic Discretization scheme. While point $c^{-1}_tX_t + c^{-1}_t\sigma_t s_t(X_t)$ is $O\left(\sqrt{d}\sigma_t/c_t\right)$ close to $X_0$, the difference $X_{t+\gamma}-c_\gamma X_t\sim \cN\left(0,\sigma_\gamma\Id_D\right)$, which is of order $O\left(\sigma_\gamma\sqrt{D}\right)$. So, for $\gamma/t \lesssim d/D$ the error scales as at least $\sigma_\gamma\sqrt{D}$. }
        \label{fig:score_matching_error}
    \end{figure}
    In Appendix~\ref{apdx:failure_of_exponential_integrator} show that $\KL$ error of exponential integrator~\eqref{eqn:reverse_sde_approx} scales at least linearly with $D$ and identify the term responsible for this. In Appendix~\ref{apdx:modified_discretization_scheme} we show that this term can be addressed by a simple linear correction and that the resulting scheme is equivalent to~\eqref{eq:discrete_version_of_final_SDE}. 
    \subsection{Failure of Exponential Integrator}
    \label{apdx:failure_of_exponential_integrator}
    Assuming the best case scenario, i.e. access to the true score $\hat{s} = s$ and absence of initialization error $\hat{Y}_0 = Y_0$ by Girsanov theorem~\cite{chen2023samplingeasylearningscore} $\KL$-divergence can be estimated 
    \begin{equation}    
    \label{eq:exponential_integrator_KL}
    \KL\left(\hat{Y}_{T-\delta}, X_{\delta}\right) \le \sum_{k = 0}^{K-1} \int^{t_{k+1}}_{t_{k}}\expectation \norm{s_{T-t}(X_{T-t}\mid T-t_k, X_{T-t_k}) - s_{T-t}(X_{T-t})}^2dt.
    \end{equation}
    
    In the following, we present a simple, informal, argument showing the right-hand side scales with $D$. To achieve this it is enough to compare $s_{t+\gamma}(X_{t+\gamma})$ with $s_{t}(X_{t})$ and show that the difference between these vectors grows scales $\sqrt{D}$. The following proposition formalizes this claim.
    \begin{proposition}
    \label{prop:growth_of_score_function_difference}
        Assume that $\gamma < t < 1/4$ and $\gamma \gtrsim d/D$ then
        \[
        \E\|s_{t+\gamma}(X_{t+\gamma}) - s_{t}(X_{t})\|^2 \gtrsim D\frac{\gamma}{t+\gamma}.
        \]
    \end{proposition}
    Substituting Proposition~\ref{prop:growth_of_score_function_difference} into~\eqref{eq:exponential_integrator_KL} we obtain either the number of steps or the error scale linearly in $D$. The rest of the section is devoted to the proof of Proposition~\ref{prop:growth_of_score_function_difference}.
    \begin{proof}
        On the one hand, by the properties of the Ornstein-Uhlenbeck process
        \begin{align*}    
        %\label{eq:Z2 as Z1 and Z3}
        &X(t) = c_{t} X_0 + \sigma_{t} Z_t,
        \nonumber \\
        &X_{t+\gamma} = c_{t+\gamma} X_0 + \sigma_{t+\gamma} Z_{t+\gamma}, 
        \\
        &X_{t+\gamma} = c_{\gamma} X_t + \sigma_\gamma Z_{\gamma}, \nonumber 
        \end{align*}
        where $Z_\gamma, Z_t, Z_{t+\gamma} \sim\cN\left (0,\Id_D\right )$ and $X_0 \sim \mu$. Note, that $\sigma_{t+\gamma} Z_{t+\gamma} = c_\gamma\sigma_{t}Z_{t} + \sigma_{\gamma} Z_{\gamma}$.
        
        On the one hand, by triangular inequality
        \[
        \norm{s_t(X_t) - s_{t+\gamma}(X_{t+\gamma})} \ge 
         \norm{\sigma_t^{-1} Z_t - \sigma^{-1}_{t+\gamma}Z_{t+\gamma}} - \norm{s_t(X_t)-\sigma^{-1}_t X_{t}} - \norm{s_{t+\gamma}(X_{t+\gamma})-\sigma^{-1}_{t+\gamma}Z_{t+\gamma}}.
        \]
        % \[
        % \norm{s(t,X(t)) - s(t+\gamma,X(t+\gamma))}^2 \ge \norm{s(t,X(t)) - X(t) +  s(t+\gamma,X(t+\gamma))}^2 
        % \]
        Finally, since $Z_t$ and $Z_\gamma$ are independent
        \[
        \sigma^{-1}_tZ_{t} - \sigma^{-1}_{t+\gamma}Z_{t+\gamma} 
        =
        \sigma^{-1}_tZ_t - \sigma^{-1}_{t+\gamma}\left(c_\gamma\sigma_tZ_t + \sigma_\gamma Z_\gamma \right) 
        \sim 
        \cN\left(0, \left(\left(\sigma^{-1}_t - \sigma^{-1}_{t+\gamma}c_\gamma\sigma_t\right)^2 + \sigma^2_\gamma\sigma^{-2}_{t+\gamma}\right)\Id_D\right).
        \]
        On the other hand, by \cite[Theorem 15]{azangulov2024convergencediffusionmodelsmanifold}, ignoring $\log$-terms, with probability at least $7/8$
        \[
        \norm{s_t(X_t) - \sigma^{-1}_tZ_{t}} \lesssim \sigma^{-1}_t\sqrt{d}.
        \]
        Therefore applying standard bounds on Normal distribution with probability at least $3/4$ holds $\norm{\sigma^{-1}_t Z_t - \sigma^{-1}_{t+\gamma}Z_{t+\gamma}} \gtrsim \sqrt{D}(\sigma_\gamma/\sigma_{t+\gamma})$. So, with probability at least $1/2$
         \[
         \norm{s(t,X(t)) - s(t+\gamma,X(t+\gamma))} \gtrsim \sqrt{D}\frac{\sigma_\gamma}{\sigma_{t+\gamma}} - \sqrt{d}/\sigma_t \gtrsim \sqrt{D\frac{\gamma}{t+\gamma}},
         \] 
         where the last holds by the choice of $t$ and $\gamma$ and since $t/4 < \sigma^2_t < 4t$.  
         By taking the square and then expectation we finish the proof. 
    \end{proof}
    \subsection{Modified Discretization Scheme}    
    \label{apdx:modified_discretization_scheme}
    As we saw in the previous section the source of the term $D$ in the difference $\norm{s_t(X_t)-s_{t+\gamma}(X_{t+\gamma})}$ is a $D$-dimensional normal vector $Z_{\gamma}$. However, it satisfies equation $X_{t+\gamma} = c_\gamma X_\gamma + \sigma_\gamma Z_\gamma$, so it can be expressed as a linear combination of $X_{t+\gamma}$ and $X_{\gamma}$ as 
    \[
    Z_{\gamma} = \frac{X_{t+\gamma}- c_\gamma X_\gamma}{\sigma_\gamma}.
    \]
    Subtracting it in the way that is demonstrated in~Figure~\ref{fig:score_matching_error} we get the representation
    \begin{equation}
    \label{eq:score_approx_in_apdx}
        (\sigma^2_t/c_t) s_t(X_t) = c^{-1}_{t+\gamma}\sigma^2_{t+\gamma}s_{t+\gamma}(X_{t+\gamma}) - (c^{-1}_tX_t - c^{-1}_{t+\gamma} X_{t+\gamma}) + O\left(\sqrt{d}\sigma_t/c_t\right).
    \end{equation}
    
    Dividing by $\sigma^2_t/c_t$ we obtain $s_t(X_t|t+\gamma, X_{t+\gamma})$ that we defined in~\eqref{eqn:first_order} in the right-hand side of~\eqref{eq:score_approx_in_apdx}. 

    Finally, substituting Tweedie's formula~\eqref{eq:tweedie_formula} we get 
    \begin{align*}    
    s_{t}(x_{t}|t',x_{t'}) 
    &=
    c^{-1}_{t'-t}\frac{\sigma^2_{t'}}{\sigma^2_{t}}s_{t'}(x_{t'}) - \frac{x_{t}-c^{-1}_{t'-t} x_{t'}}{\sigma^2_{t}} 
    \\
    &= 
    c^{-1}_{t'-t}\frac{\sigma^2_{t'}}{\sigma^2_{t}}\frac{c_t\expectation[X_0|X_t=x] - x}{\sigma^2_t} - \frac{x_{t}-c^{-1}_{t'-t} x_{t'}}{\sigma^2_{t}} 
    \\
    &= \frac{c_{t}\expectation[X_0|X_{t'}=x_{t'}] - x_{t}}{\sigma^2_{t}}.
    \end{align*}
    Finally, we prove Proposition~\ref{prop:equiv_to_lie}
    \EQUIVALENCETOLIE*    
    We now show that the iteration in~\eqref{eq:discrete_version_of_final_SDE} is precisely the solution of the piecewise-linear SDE~\eqref{eq:approx_backward_process_discrete_modified} integrated over each sub-interval. Fix \(k\) and consider \(t \in [t_k,t_{k+1})\). On this interval, the SDE~\eqref{eq:approx_backward_process_discrete_modified} is a linear SDE in \(\hat{Y}_t\).  In particular, one can rewrite it as 
    \[
    d\hat{Y}_t 
    = 
    \Bigl[\,
    A(t)\,\hat{Y}_t
    \;+\;
    \alpha(t)
    \Bigr]\,dt
    \;+\;
    \sqrt{2}\,dB_t',
    \]
    where
    \[
    A(t) 
    =
    1 - \frac{2}{\sigma_{T-t}^2},
    \qquad
    \alpha(t)
    = 2 c_{t-t_k}^{-1}\frac{\sigma_{T-t_k}^2}{\sigma_{T-t}^2}
    \hat{s}_{T-t_k}(\hat{Y}_{t_k})
    + 2c_{t-t_k}^{-1}\frac{\hat{Y}_{t_k}}{\sigma_{T-t}^2}.
    \]
    This can be solved explicitly via the standard variation-of-constants formula which gives
    \begin{align*}
    \hat{Y}_{t_{k+1}}
    &=\Phi_{t_{k+1}}\left [\hat{Y}_{t_k}
    +
    \int_{t_k}^{t_{k+1}}\Phi_{s}^{-1}\alpha(s)ds +
    \sqrt{2}\int_{t_k}^{t_{k+1}}\Phi_{s}^{-1}dB_s' \right],
    \end{align*}
    where
    \begin{equation*}
        \Phi_s = \exp \left (\int_{t_k}^sA(u)du \right ).
    \end{equation*}
    Using \(c_t=e^{-t}\) and  \(\sigma_t^2 = 1-e^{-2t}\) we can evaluate the above integrals as follows
    \begin{align*}        \Phi_s 
    &= c_{s-t_k}^{-1} \\             \int_{t_k}^{t_{k+1}}\Phi_{s}^{-1}\alpha(s)ds
        &= \sigma_{\gamma_k}^2
        \hat{s}_{T-t_k}\bigl(\hat{Y}_{t_k}\bigr),\\
        \sqrt{2}\int_{t_k}^{t_{k+1}}\Phi_{s}^{-1}dB_s'
        &= c_{\gamma_k}\sigma_{\gamma_k}\,\frac{\sigma_{T - t_{k+1}}}{\sigma_{T - t_k}}\,Z_k,
        \quad
        Z_k \sim \cN\bigl(0,I_D\bigr).
    \end{align*}
    Putting these together yields
    \[
    \hat{Y}_{t_{k+1}}
    \;=\;
    c_{\gamma_k}^{-1}\,\hat{Y}_{t_k}
    \;+\;\frac{\sigma_{\gamma_k}^2}{c_{\gamma_k}}\;\hat{s}_{T-t_k}\bigl(\hat{Y}_{t_k}\bigr)
    \;+\;
    \sigma_{\gamma_k}\,\frac{\sigma_{T - t_{k+1}}}{\sigma_{T - t_k}}\,Z_k,
    \]
    which coincides with~\eqref{eq:discrete_version_of_final_SDE}.
       
\section{Proofs of Lemmas}
    \label{apdx:proofs}
    
    \ReverseMartingale*
    \begin{proof}\label{proof:reverse_martingale}
        Since $X_0 \in L^1$, $\curly{\E[X_0 | \mathcal{F}_t]}_{t \in [0,T]}$ is a Doob martingale. Since $\{X_t\}_{t \in [0,T]}$ is a Markov process, we have $\E[X_0 | \mathcal{F}_t] = \E[X_0 | X_t] = m_t(X_t)$. This completes the proof.
    \end{proof}
    
    \MartingaleIncrements*
    \begin{proof}\label{proof:martingale_increments}
        We follow the proof of \cite[Proposition 3.14]{LeGall} with minimal modifications for the case when $M_t$ takes values in $\R^D$. 
        \begin{align*}
            \E \|M_{t_2}-M_{t_{1}}\|^2
            &= \E \left [\E \left [\|M_{t_2}-M_{t_{1}}\|^2|\mathcal{F}_{t_1}\right]\right]
            = \E\left [\E \left [\|M_{t_2}\|^2 - 2 \langle M_{t_2}, M_{t_1}\rangle + \|M_{t_{1}}\|^2|\mathcal{F}_{t_1}\right ]\right]\\
            &= \E\|M_{t_2}\|^2 - \E\|M_{t_{1}}\|^2.
        \end{align*}
        Applying the same calculation to $\E \|M_{t_3}-M_{t_{1}}\|^2$ and $\E \|M_{t_3}-M_{t_{2}}\|^2$, then summing up gives the desired result.
    \end{proof}

    \Decoupling*
    \begin{proof}\label{proof:Decoupling}    
        We introduce the process $\curly{\hat{Y}_t'}_{t\in [0,T-\delta]}$ that approximates the true backwards process $\curly{Y_t}_{t\in [0,T-\delta]}$ and is given by
        \begin{equation*}    
            \begin{cases}
                d\hat{Y}'_t = \square{\hat{Y}'_t + 2\hat{s}_{T-t}(\hat{Y}_t|T-t_k,\hat{Y}_{t_k})}dt+ \sqrt{2}dB_t,
                 &t\in [t_k,t_{k+1})
                \\
                \hat{Y}'_0 \stackrel{dist.}{=} X_T. &
            \end{cases}
        \end{equation*}
        $\hat{Y}'_t$ follows the same dynamics as the process $\hat{Y}_t$, but it is initialized with the true distribution $\hat{Y}'_{0} \stackrel{dist.}{=} X_T$, not by Gaussian noise $\hat{Y}_0 \sim \cN\left(0,\Id_D\right)$. By~\cite[Section 3.3]{benton2024nearly} and the data-processing inequality   
        \[
        \KL(\hat{Y}_{T-\delta}\|  X_{\delta}) \le \KL(\hat{Y}\|  Y) = \KL(\hat{Y}'\| Y) + \KL\left(Y_0\|\cN\left(0,\Id_D\right)\right), 
        \]
        where $\KL(X\| Y)$ is the $\KL$-divergence between the path measures of processes $X$ and $Y$. We bound the terms separately. Combining Assumption~\ref{asmp:2} with \cite[Proposition 5]{benton2024nearly} we have
        \begin{align}
        \label{eq:lemma13_1}
            \KL\left(Y_0\|\cN\left(0,\Id_D\right)\right) \lesssim \left(D+\E\norm{X_0}^2\right)e^{-2T} \le (D+1)e^{-2T} \lesssim De^{-2T}.
        \end{align}
        At the same time, by~\cite[Proposition 3]{benton2024nearly}, we have a Girsanov-like bound
        \[
        \KL(\hat{Y}'\|  Y) \le \sum_{k = 0}^{K-1} \int^{t_{k+1}}_{t_{k}}\expectation \norm{\hat{s}_{T-t}(X_{T-t}\mid T-t_k, X_{T-t_k}) - s_{T-t}(X_{T-t})}^2 dt.
        \]
        Applying triangular inequality we estimate from above
        \begin{multline}    
        \label{eq:triangular_for_integrals}
        \sum_{k = 0}^{K-1} \int^{t_{k+1}}_{t_{k}}\expectation \norm{\hat{s}_{T-t}(X_{T-t}\mid T-t_k, X_{T-t_k}) - s_{T-t}(X_{T-t})}^2 dt
        \\
        \le
        2\sum_{k = 0}^{K-1} \int^{t_{k+1}}_{t_{k}}\expectation \norm{\hat{s}_{T-t}(X_{T-t}\mid T-t_k, X_{T-t_k}) - s_{T-t}(X_{T-t}\mid T-t_k, X_{T-t_k})}^2 dt 
        \\
        +
        2\sum_{k = 0}^{K-1} \int^{t_{k+1}}_{t_{k}}\expectation \norm{s_{T-t}(X_{T-t}\mid T-t_k, X_{T-t_k}) - s_{T-t}(X_{T-t})}^2 dt. 
        \end{multline}
        To bound the first sum we recall that by the design of discretization scheme $t_{k+1}-t_k = \kappa\min(1,T- t_k)$. Let $t\in [t_k,t_{k+1}]$.
        First, since $t_{k+1}-t_k \le \kappa < 1/4$ we have  $c^{-1}_{t-t_k} \le c^{-1}_{\kappa} =  e^{\kappa} \le  2$. Second, since $(T-t_k)/(T-t) \in [1, 1+\kappa]$ we have $\sigma^2_{T-t_k}/\sigma^2_{T-t} \le 2$. Multiplying, we get 
        $
        \sup_{t\in [t_k, t_{k+1}]} c^{-1}_{t-t_k}\frac{\sigma^2_{T-t_k}}{\sigma^2_{T-t}} \le 4.
        $
        Therefore, applying~\eqref{eqn:first_order} we bound the first sum as
        \begin{multline}   
        \label{eq:bound_on_adjusted_score_integral}
        \sum_{k = 0}^{K-1} \int^{t_{k+1}}_{t_{k}}\expectation \norm{\hat{s}_{T-t}(X_{T-t}\mid T-t_k, X_{T-t_k}) - s_{T-t}(X_{T-t}\mid T-t_k, X_{T-t_k})}^2 dt 
        \\
        =
        \sum_{k = 0}^{K-1} \int^{t_{k+1}}_{t_{k}}\expectation c^{-1}_{t-t_k}\frac{\sigma^2_{T-t_k}}{\sigma^2_{T-t}}\norm{\hat{s}_{T-t_k}(X_{T-t_k}) - s_{T-t_k}(X_{T-t_k})}^2 dt \le 4\varepsilon^2_{score}.
        \end{multline}
        So, substituting~\eqref{eq:bound_on_adjusted_score_integral} into~\eqref{eq:triangular_for_integrals} we get
        \begin{multline}
        \label{eq:lemma13_2}
        \sum_{k = 0}^{K-1} \int^{t_{k+1}}_{t_{k}}\expectation \norm{\hat{s}_{T-t}(X_{T-t}\mid T-t_k, X_{T-t_k}) - s_{T-t}(X_{T-t})}^2dt
        \\
        \le 
        8\varepsilon^2_{score} + 
        2\sum_{k = 0}^{K-1} \int^{t_{k+1}}_{t_{k}}\expectation \norm{s_{T-t}(X_{T-t}\mid T-t_k, X_{T-t_k}) - s_{T-t}(X_{T-t})}^2 dt.
\end{multline}
        Combining~\eqref{eq:lemma13_1} and \eqref{eq:lemma13_2} completes the proof.
    \end{proof}

    \Monotonicity*
    \begin{proof}\label{proof:monotonicity}    
        We recall \eqref{eq:tweedi_corollary} which represents the error as a product of two terms. Since both terms are positive, it is enough to prove that both terms are decreasing. First, a simple calculation shows that for $t \ge 0$
        \[
        \frac{d}{dt}\left(c_t^2/\sigma^4_t\right) = \frac{d}{dt}\frac{e^{-2t}}{(1-e^{-2t})^2} = -\frac{2e^{2t}(e^{2t}+1)}{(e^{2t}-1)^3} <  0.
        \]
        So $c_t^2/\sigma^4_t$ is decreasing and $c^2_{t_1}/\sigma^4_{t_1} \geq c^2_{t_2}/\sigma^4_{t_2}$. Second, by Lemma~\ref{lemma:orth_increments}, the second term is decreasing since
        \begin{align*}
            \E\norm{m_{t_1}(X_{t_1}) - m_{t'}(X_{t'})}^2 &=  \E\norm{m_{t_1}(X_{t_1}) - m_{t_2}(X_{t_2})}^2 + \E\norm{m_{t_2}(X_{t_2}) - m_{t'}(X_{t'})}^2 \\
            &\ge \E\norm{m_{t_2}(X_{t_2}) - m_{t'}(X_{t'})}^2.
        \end{align*}
        Combining these two statements, we get the lemma.
    \end{proof}
    \ManifoldConcentrationBound*
    \begin{proof} \label{proof:manifold_concentration_bound}
        First, we note that $m_t(x) = \int_M y \mu_{0|t}(dy|x)$ where we recall that $\mu_{0|t}(dy|x)$ is the law of $\left(X_0|X_t=x\right)$ and so
        \[
        X_0 - m_{t}(X_{t}) = \int_M \left(X_0-y\right)\mu_{0|t}(dy|X_t).
        \]
        Since $\diam M \le 1$ and $X_0 \in M$, a.s.\ we have that $\|X_0 - m_{t}(X_{t})\| \le 1$. Therefore, it is enough to consider the case $t\le d^{-1} \le 1$.  In this case, $(\sigma_t/c_t)^2 \simeq t$. We also note that
        \[
        \|X_0 - m_{t}(X_{t})\|^2 \le \int_M \norm{X_0-y}^2\mu_{0|t}(dy|X_t).
        \] 
        As shown in \cite[Theorem 15]{azangulov2024convergencediffusionmodelsmanifold}, 
        if 
        \[
        r(t,\delta, \eta) = 2(\sigma_t/c_t)\sqrt{20d\left(\log_+ (\sigma_t/c_t)^{-1} + 4C\right) + 8\log \delta^{-1} + \log \eta^{-1}},
        \]
        then with probability at least $1-\delta$
        \[
        \int_{y\in M: \norm{X_0-y}\le r(t,\delta,\eta)}  \mu_{0|t}(dy|X_t) \ge  1-\eta.
        \]
        Integrating $\|X_0-y\|^2$ w.r.t.\ $\mu_{0|t}(dy|X_t)$ we get that with probability at least~${1-\delta}$
        \[
        \|X_0 - m_{t}(X_{t})\|^2 \le 
        \int_M \norm{X_0-y}^2\mu_{0|t}(dy|X_t) \le r^2(t,\delta,\eta) + \eta.
        \]
        Taking an expectation w.r.t.\ both $X_0$ and $X_t$ we get 
        \[
        \expectation{\|X_0 - m_{t}(X_{t})\|}^2 \le r^2(t,\delta,\eta) + \eta + \delta.
        \]
        Choosing $\delta = \eta = \min\left(1, (\sigma_t/c_t)^2\right)$, we have
        \begin{align*}
            \expectation{\|X_0 - m_{t}(X_{t})\|}^2 
            &\le 4(\sigma_t/c_t)^2\left(20d\left(\log_+ (\sigma_t/c_t)^{-1} + 4C\right) + 18\log_+ (\sigma_t/c_t)^{-1}\right). 
        \end{align*}
        Therefore
        \[
        \expectation{\|X_0 - m_{t}(X_{t})\|}^2 \lesssim (\sigma_t/c_t)^2d\left(\log_+ (\sigma_t/c_t)^{-1} + C\right).
        \]
        As we mentioned, we can limit ourselves to the case $t\le 1$. For $t \in [\delta, 1]$, we have $(\sigma_t/c_t)^2 \simeq t$ and $\log_+(\sigma_t/c_t)^{-1} \lesssim \log \delta^{-1}$. This shows that
         \[
         \expectation{\|X_0 - m_{t}(X_{t})\|}^2 \lesssim td\left(\log\delta^{-1} + C\right).
         \]
    \end{proof}

\end{document}